\documentclass[letterpaper]{article}
\usepackage{uai2019}
\usepackage{times}
\usepackage[margin=1in]{geometry}
\usepackage{hyperref}       
\usepackage{mathtools}
\usepackage{microtype}
\usepackage{graphicx}
\usepackage{subfigure}
\usepackage{booktabs} 
\usepackage[round, sort]{natbib}
\usepackage{mathrsfs}
\usepackage{xcolor}
\usepackage[algo2e]{algorithm2e} 
\usepackage{algorithm,algorithmic}
\usepackage{enumitem}

\usepackage{amsmath,amsfonts,amsthm,bm}
\usepackage{bbm}

\newcommand{\comment}[1]{}

\newtheorem{lemma}{Lemma}

\def\eqref#1{equation~\ref{#1}}

\def\1{\mathbbm{1}}

\def\vtheta{{\bm{\theta}}}
\def\va{{\bm{a}}}

\def\vw{{\bm{w}}}
\def\vx{{\bm{x}}}

\DeclareMathAlphabet{\mathsfit}{\encodingdefault}{\sfdefault}{m}{sl}
\SetMathAlphabet{\mathsfit}{bold}{\encodingdefault}{\sfdefault}{bx}{n}

\def\tC{{\tens{C}}}

\newcommand{\E}{\mathbb{E}}

\DeclareMathOperator*{\argmin}{arg\,min}

\newcommand{\ie}{\emph{i.e.}}
\newcommand{\eg}{\emph{e.g.}}
\newcommand{\versus}{\emph{v.s.}}
\renewcommand{\eqref}[1]{Eq. (\ref{#1})}

\newcommand\independent{\protect\mathpalette{\protect\independenT}{\perp}} 
\def\independenT#1#2{\mathrel{\rlap{$#1#2$}\mkern2mu{#1#2}}}

\definecolor{lowblue}{rgb}{0.158, 0.288, 0.72}
\definecolor{lowred}{rgb}{0.758, 0.188, 0.178}
\newcommand\boldblue[1]{\textcolor{lowblue}{\textbf{#1}}}
\newcommand\boldred[1]{\textcolor{lowred}{\textbf{#1}}}

\newtheorem{proposition}{Proposition}
\newtheorem{remark}{Remark}

\theoremstyle{definition}

\newcommand{\W}{\mathcal{W}}
\newcommand{\vA}{\bm{A}}
\newcommand{\vX}{\bm{X}}
\newcommand{\Sa}{S_{\va}}
\newcommand{\bp}{p_{\bar{S}}}
\newcommand{\bs}{\bar{s}}
\newcommand{\bS}{\bar{S}}
\newcommand{\tc}{\textrm{c}}
\renewcommand{\tC}{\textrm{C}}
\newcommand{\st}{\emph{s.t.}}

\title{Wasserstein Fair Classification}

\author{{\bf Ray Jiang\thanks{\,\, Equal contribution.}} \\
DeepMind \\
rayjiang@google.com\\
\And
{\bf Aldo Pacchiano\footnotemark[1]}  \\
UC Berkeley, DeepMind \\
pacchiano@berkeley.edu 
\\
\And
{\bf Tom Stepleton}   \\
DeepMind \\
stepleton@google.com    \\
\AND
{\bf Heinrich Jiang}   \\
Google Research \\
heinrichj@google.com    \\
\And
{\bf Silvia Chiappa}   \\
DeepMind \\
csilvia@google.com    \\
}

\begin{document}
\maketitle

\begin{abstract}
We propose an approach to fair classification that enforces independence between the classifier outputs and sensitive information by minimizing Wasserstein-1 distances. The approach has desirable theoretical properties and is robust to specific choices of the threshold used to obtain class predictions from model outputs. We introduce different methods that enable hiding sensitive information at test time or have a simple and fast implementation. We show empirical performance against different fairness baselines on several benchmark fairness datasets.
\end{abstract}

\section{INTRODUCTION}
The{\let\thefootnote\relax\footnote{{In Proceedings of the Thirty-Fifth Conference on Uncertainty in Artificial Intelligence, 2019. Code available at \href{https://github.com/deepmind/wasserstein_fairness}{github.com/deepmind/wasserstein\_fairness}.}}}
increasing use of machine learning in decision-making scenarios that have serious implications for individuals and society, such as health care, criminal risk assessment, social services, hiring, financial lending, and online advertising \citep{defauw18clinically, dieterich16compas, eubanks18automating, hoffman18discretion, malekipirbazari15risk, perlich14machine}, is raising concern that bias in the data and model inaccuracies can lead to decisions that are ``unfair'' towards underrepresented or historically discriminated groups. 

This concern has motivated researchers to investigate ways of ensuring that sensitive information (\eg~race and gender) does not ‘unfairly’ influence the decisions. In the classification case considered in this paper, the most widely used approach is to enforce statistical independence between class predictions and sensitive attributes, a criterion called demographic parity \citep{feldman15certifying}. 

In the common scenario in which the model outputs continuous values from which class predictions are obtained through thresholds, this approach would however ensure fairness only with respect to the particular choice of thresholds. 
Furthermore, as independence constraints on the class predictions are difficult to impose in practice, uncorrelation constraints on the model outputs are often imposed instead.

In this paper, we propose an approach that overcomes these limitations by imposing independence constraints directly on the model outputs. This is achieved through enforcing small Wasserstein distances between the distributions of the model outputs corresponding to groups of individuals with different sensitive attributes. We demonstrate that using Wasserstein-1 distances to the barycenter is optimal, in the sense that it achieves independence with minimal changes to the class predictions that would have been obtained without constraints. 
We introduce a Wasserstein-1 penalized logistic regression method that learns the optimal transport map in the logistic model parameters, with a variation that has the advantage of being demographically blind at test time. In addition, we provide a simpler and faster post-processing method. We show that the proposed methods outperform previous approaches in the literature on four benchmark fairness datasets. 

\section{STRONG DEMOGRAPHIC PARITY}
Let $\mathcal{D} =\{(\va^n,\vx^n,y^n)\}_{n=1}^N$ be a sequence of $N$ \emph{i.i.d.} samples drawn from an unknown probability distribution over $\mathcal{A} \times \mathcal{X} \times \mathcal{Y}=\mathbb{N}^k\times\mathbb{R}^d\times\{0, 1\}$. Each datapoint $(\va^n,\vx^n,y^n)$ corresponds to information from an individual (or community): $y^n$ 
indicates a binary class, each element $a^n_i$ of $\va^n$ corresponds to a different sensitive attribute, \eg~to the gender of the individual, and 
$\vx^n$ is a feature vector that, possibly together with $\va^n$, can be used to form a prediction $\hat y^n\in\{0,1\}$ of the class $y^n$. 
We denote with $\mathcal{D}_{\va} = \{ (\va^n, \vx^n, y^n) \in \mathcal{D}~\st~\va^n = \va \}$ the set of $N_{\va}$ individuals belonging to group $\va$. We indicate with $\vA,\vX,Y$ and $\hat Y$ the random variables corresponding to $\va^n,y^n$ and $\hat y^n$, and with $p(\cdot)$ or $p_X(\cdot)$ probability density functions (pdfs), where the latter is used to emphasize the associated random variable. 

Many classifiers, rather than a binary class prediction $\hat y^n$, output a non-binary value $s^n$. In the logistic regression case considered in this paper, $s^n\in\Omega=[0,1]$ indicates the \emph{model belief} that individual $n$ belongs to class 1, \ie~$s^n=\mathbbm{P}(Y=1|\vA = \va^n,\vX=\vx^n)$\footnote{Throughout the paper, we use $\mathbbm{P}(\cdot)$ to indicate probability measures associated with the corresponding probability spaces $(O, \mathcal{F}, \mathbbm{P}(\cdot))$ where $\mathcal{F}$ is a $\sigma$-algebra on the sample output space $O$.}.
From $s^n$, a class prediction $\hat y^n\in\{0,1\}$ is obtained using a threshold $\tau\in\Omega$, \ie~$\hat y^n := \1_{s^n>\tau}$, where $\1_{s^n>\tau}$ equals to one if $s^n>\tau$ and zero otherwise. We call the random variable $S$ corresponding to $s^n$ the \emph{belief variable}, and 
denote with $S_{\va}$ the belief variable for group $\va$, \ie~with pdf $p(S_\va)=p(S|\vA=\va)$. 

We are interested in ensuring that sensitive information does not influence the decisions. 
This is often achieved by imposing that the model satisfies a fairness criterion called \emph{demographic parity (DP)}, defined as
\begin{align*}
\mathbbm{P}(\hat Y = 1 | \vA = \va) = \mathbbm{P}(\hat Y = 1 | \vA = \bar\va)\,, \hskip0.2cm \forall \va,\bar\va \in \mathcal{A} \,.
\end{align*}
DP can equivalently be expressed as requiring statistical independence between $\hat Y$ and $\vA$, denoted as $\hat Y \independent \vA$.

Enforcing demographic parity at a given threshold $\tau$ does not necessarily imply that the criterion is satisfied for other thresholds. 
Furthermore, to alleviate difficulties in optimizing on the class prediction $\hat Y$, relaxations are often considered, such as imposing the constraint $\E[S|\vA=\va]=\E[S| \vA = \bar\va]$ $\forall \va,\bar\va \in \mathcal{A}$, where $\E[\cdot]$ denotes expectation \citep{goh16satisfying,zafar17fairness}.

To deal with these limitations, we propose an approach that enforces statistical independence between $S$ and $\vA$, $S \independent \vA$. We call this fairness criterion \emph{strong demographic parity (SDP)}, as it ensures that the decision does not depend on the sensitive attribute regardless of the threshold $\tau$ used, since $S\independent \vA$ implies $\hat Y\independent \vA$ for any value of $\tau$. 
SDP can be defined as 
\begin{align*}
\hskip0.2cm p_{\Sa} = p_{S_{\bar\va}}, \hskip0.3cm \forall \va,\bar\va\in \mathcal{A} \,.
\end{align*}

In Remark~\ref{remark:1}, we prove that this definition is equivalent to\footnote{We omit the brackets from the expectation to simplify the notation.} 
$$\mathbb{E}_{\tau\sim U(\Omega)} |\mathbbm{P}(S_{\va}>\tau) - \mathbbm{P}(S_{\bar\va}>\tau)| = 0, \hskip0.2cm\forall \va,\bar\va\in \mathcal{A}\,,$$
where $U(\Omega)$ denotes the uniform distribution over $\Omega$. This result leads us to use 
\begin{align*}
\sum_{\substack{\va,\bar\va\in \mathcal{A} \\ \st~\va\neq \bar\va}} \mathbb{E}_{\tau\sim U(\Omega)} |\mathbbm{P}(S_{\va}>\tau) - \mathbbm{P}(S_{\bar\va} > \tau)|\,,
\end{align*}
as a measure of dependence of $S$ on $\vA$, the we call \emph{strong pairwise demographic disparity (SPDD)}.  

\section{WASSERSTEIN FAIR CLASSIFICATION}
\vskip-0.2cm
We suggest to achieve SDP by enforcing the model output pdfs corresponding to groups of individuals with different sensitive attributes, $\{p_{\Sa}\}_{\va\in \mathcal{A}}$,  to coincide with their Wasserstein-1 barycenter distribution $p_{\bS}$. 
The use of the Wasserstein distance is motivated because this distance is defined and computable even between distributions with disjoint supports. This is critical because the empirical estimates $\{\hat{p}_{S_a}\}$, $\hat{p}_{\bS}$ of 
$\{p_{S_a}\}$ and $p_{\bS}$ used to implement the methods and their supports are typically disjoint.

\subsection{OPTIMALITY OF WASSERSTEIN-1 DISTANCE\label{sec:optimality}}
\paragraph{Preliminary.} Given two pdfs $p_{X}$ and $p_{Y}$ on ${\cal X}$ and ${\cal Y}$, a transportation map $T: {\cal X}\rightarrow {\cal Y}$ is defined by $\int_{\mathcal{B}}p_{Y}(y)dy = \int_{T^{-1}(\mathcal{B})}p_{X}(x)dx$ for any measurable subset $\mathcal{B}\subset {\cal Y}$ (indicating that the mass of the set $\mathcal{B}$ with respect to the density $p_Y$ equals the mass of the set $T^{-1}(\mathcal{B})$ with respect to the density $p_{X}$). Let $\mathcal{T}$ be the set of transportation maps from ${\cal X}$ to ${\cal Y}$, and $\tc:{\cal X}\times {\cal Y}\rightarrow[0,\infty]$ be a cost function such that $\tc(x, T(x))$ indicates the cost of transporting $x$ to $T(x)$. In the original formulation \citep{monge81memoire}, the optimal transport map $T^*$ is the one that minimizes the total transportation cost, \ie 
\[
T^* = \argmin_{T\in\mathcal{T}} \int_{x\in {\cal X}} \tc(x, T(x)) p_{X}(x) dx.
\]
To address limitations of this formulation, \citet{kantorovich42on} reformulated the optimal transport problem as finding an optimal pdf $p_{X\times Y}$ in 
the set $\Gamma(p_X, p_Y)$ of joint pdfs on ${\cal X}\times {\cal Y}$ with marginals over $Y$ and $X$ given by $p_X$ and $p_Y$ such that
\[
\gamma^* = \argmin_{\gamma\in\Gamma(f_X, f_Y)} \int_{{\cal X}\times {\cal Y}} \tc(x, y) p_{X\times Y}dxdy.
\]

The $p$-Wasserstein distance is defined as
\[
\mathcal{W}_p(p_X, p_Y) = \hskip-0.1cm\min_{\gamma\in\Gamma(f_X, f_Y)} \left(\int_{{\cal X}\times {\cal Y}} \hskip-0.1cm\textrm{d}(x,y)^p p_{X\times Y}dxdy\right)^{\frac{1}{p}}\hskip-0.1cm,
\]
where ${\cal X}= {\cal Y}$, d is a distance on ${\cal X}$, and $p\geq 1$. 

\paragraph{Fair Optimal Post-Processing.}
Let us first consider the problem of post-processing the beliefs of a model to achieve SDP while making minimal model class prediction changes.

Let $S_1$ and $S_2$ be two belief variables with values in $\Omega=[0,1]$ and pdfs $p_{S_1}$ and $p_{S_2}$, and let $T: \Omega \rightarrow \Omega$ be a transportation map satisfying $\int_{\mathcal{B}}p_{S_2}(y)dy = \int_{T^{-1}(\mathcal{B})}p_{S_1}(x)dx$ for any measurable subset $\mathcal{B}\subset {\Omega}$. Let $\mathcal{T}$ be the set of all such transportation maps. A class prediction $\hat y=\1_{s_1>\tau}$ changes due to transportation $T(s_1)$ if and only if $\tau\in(m_{s_1}^T, M_{s_1}^T)$ where $m_{s_1}^T = \min[s_1, T(s_1)]$ and $M_{s_1}^T = \max [s_1, T(s_1)]$. This observation leads to the following result.

\begin{proposition}
Given two belief variables $S_1$ and $S_2$ in $\Omega=[0,1]$ with pdfs $p_{S_1}$ and $p_{S_2}$, the following three quantities are equal:
\begin{enumerate}
\item[\emph{(i)}] $\W_1(p_{S_1}, p_{S_2}) = \min\limits_{T\in\mathcal{T}} \int_{x\in\Omega} |x-T(x)|p_{S_1}(x)dx$.
\item[\emph{(ii)}]
$\mathbb{E}_{\tau\sim U(\Omega)} |\mathbbm{P}(S_1>\tau) - \mathbbm{P}(S_2 > \tau)|$.
\item[\emph{(iii)}] Expected class prediction changes due to transporting $p_{S_1}$ into $p_{S_2}$ through the map $T^*$ \[\mathbb{E}_{\tau\sim U(\Omega), x\sim p_{S_1}} \mathbbm{P}(\tau\in(m_x^{T^*}, M_x^{T^*}))\,.\] 
\end{enumerate}
\label{prop:1}
\end{proposition}
\begin{proof}
In the one-dimensional case of $p_{S_1}$ and $p_{S_2}$, the total transportation cost $\W_1(p_{S_1}, p_{S_2})$ can be written as 
\begin{align*}
\W_1(p_{S_1}, p_{S_2}) &=\text{\footnotemark} \int_{x=0}^1 |P_{S_1}^{-1}(x) - P_{S_2}^{-1}(x)|dx \\
&=\int_{\tau=0}^1 |P_{S_1}(\tau) - P_{S_2}(\tau)|d\tau \tag*{(by Lemma~\ref{lemma::inverse_cdf} in Appendix C)}\\
&= \mathbb{E}_{\tau\sim U(\Omega)} |\mathbbm{P}(S_1\leq \tau) - \mathbbm{P}(S_2 \leq \tau)|\\
&= \mathbb{E}_{\tau\sim U(\Omega)} |\mathbbm{P}(S_1>\tau) - \mathbbm{P}(S_2 > \tau)|\,,    
\end{align*}
where $P_{S_1}$ and $P_{S_2}$ are the cumulative distribution functions of $S_1$ and $S_2$ respectively. This prove that (i) equals (ii).
\footnotetext{The proof of this equality can be found in \citet{rachev98mass}.}

The expected class prediction changes due to applying the transportation map $T$ is given by
\begin{align*}
&\hskip-0.4cm\mathbb{E}_{\substack{\tau\sim U(\Omega)\\[2pt] x\sim p_{S_1}}} \mathbbm{P}(\tau\in(m_x^T, M_x^T))\nonumber\\
& \hskip2cm = \int_{\tau=0}^1 \int_x |x-T(x)|p_{S_1}(x)dx d\tau \nonumber\\
& \hskip2cm = \int_x |x-T(x)|p_{S_1}(x)dx. 
\end{align*}
Thus, 
\begin{align*}
\hskip-0.0cm\W_1(p_{S_1}, p_{S_2}) &= \min_{T\in\mathcal{T}} \int_x |x-T(x)|p_{S_1}(x)dx \\
&\hskip-0.0cm=\int_x |x-T^*(x)|p_{S_1}(x)dx \\
&\hskip-0.0cm =\mathbb{E}_{\substack{\tau\sim U(\Omega)\\[2pt] x\sim p_{S_1}}} \mathbbm{P}(\tau\in(m_x^{T^*}, M_x^{T^*}))\,.
\end{align*}
This prove that (i) equals (iii).
\end{proof}

\begin{remark}\label{remark:1}
Notice that $\textrm{(ii)}=\mathbb{E}_{\tau\sim U(\Omega)} |\mathbbm{P}(S_1 > \tau) - \mathbbm{P}(S_2 > \tau)| = 0$ if and only if $p_{S_1} = p_{S_2}$. Indeed, by Proposition~\ref{prop:1} and the property of the $\W_1$ metric, $\textrm{(ii)}=0$ $\iff \W_1(p_{S_1}, p_{S_2}) = 0 \iff p_{S_1} = p_{S_2}$.
\end{remark}

To reach SDP, we need to achieve $p_{\Sa} = p^*$ $\forall \va \in \mathcal{A}$, where $p^*\in\mathcal{P}(\Omega)$, the space of pdfs on $\Omega$. We would like to choose transportation maps $T$ and a target distribution $p^*$ such that the transportation process from $p_{\Sa}$ to $p^*$ incurs minimal total expected class prediction changes. Assume that the groups are all disjoint, so that the per-group transportation maps $T$ are independent from each other. Let $\mathbbm{T}(p^*)$ be the set of transportation maps with elements $T$ such that, restricted to group $\va$, $T$ is a transportation map from $p_{\Sa}$ to $p^*$ (\ie~
$\mathbbm{T}(p^*)=\{T\in\mathcal{T}(p_S, p^*) \,\mid\, T(S)\big|_{A=\va} = T_\va\in\mathcal{T}_{\va} = \mathcal{T}(p_{\Sa}, p^*)\}$ 
where $\mathcal{T}(p_S, p^*)$ denotes the space of transportation maps from $p_S$ to $p^*$). We would like to obtain
\begin{align*}
    &\hskip-0.1cm\min_{\substack{T\in\mathbbm{T}(p^*)\\ p^*\in \mathcal{P}(\Omega)}} \mathop{\mathbb{E}}_{\substack{\tau\sim U(\Omega)\\[2pt] x\sim p_S}}\mathbbm{P}(\tau\in(m_x^T, M_x^T))\\
    &=\min_{\substack{T\in\mathbbm{T}(p^*)\\ p^*\in \mathcal{P}(\Omega)}} \sum_{\va\in\mathcal{A}} \underbrace{p(A=\va)}_{p_{\va}}\mathop{\mathbb{E}}_{\substack{\tau\sim U(\Omega)\\[2pt] x\sim p_{\Sa}}}\mathbbm{P}(\tau\in(m_x^T, M_x^T))\\
    & =\min_{p^*\in \mathcal{P}(\Omega)} \sum_{\va\in\mathcal{A}}p_{\va}\min_{T\in\mathcal{T}_{\va}}\mathop{\mathbb{E}}_{\substack{\tau\sim U(\Omega)\\[2pt] x\sim p_{\Sa}}} \mathbbm{P}(\tau\in(m_x^T, M_x^T))\\
    & = \min_{p^*\in \mathcal{P}(\Omega)} \sum_{\va\in\mathcal{A}}p_{\va} \min_{T\in\mathcal{T}_{\va}} \int_{x \in \Omega} |x-T(x)|p_{\Sa}(x)dx\\ 
    & = \min_{p^*\in \mathcal{P}(\Omega)} \sum_{\va\in\mathcal{A}} p_{\va}\W_1(p_{\Sa}, p^*)\,.
\end{align*}
Therefore we are interested in
\begin{align}\label{eq:bar}
\bp = \argmin_{p^* \in \mathcal{P}(\Omega)} \sum_{\va \in \mathcal{A}} p_{\va}\W_1(p_{\Sa}, p^*)\,,
\end{align}
which coincides with the Wasserstein-1 barycenter with normalized subgroup size as weight to every group distribution $p_{\Sa}$ \citep{agueh11barycenters}. 

In summary, we have demonstrated that the optimal post-processing procedure that minimizes total expected model prediction changes is to use the Wasserstein-1 optimal transport map $T^*$ to transport all group distributions $p_{\Sa}$ to their weighted barycenter distribution $p_{\bS}$. 

\paragraph{Optimal Trade-Offs.}
We have shown that post-processing the beliefs of a model through optimal transportation achieves SDP (and therefore $\textrm{SPDD}=0$) whilst minimizing expected prediction changes. 
We now examine the case in which, after transportation, SDP is not attained, \ie~SPDD is positive.
By triangle inequality
\begin{align*}
    \text{SPDD} 
    & \leq 2(|\mathcal{A}|\!-\!1) \sum_{\va\in \mathcal{A}} \mathbb{E}_{\tau\sim U(\Omega)} |\mathbbm{P}(S_{\va}>\tau) \!-\! \mathbbm{P}(\bS>\tau)|\\
    & = 2(|\mathcal{A}|\!-\!1) \sum_{\va\in \mathcal{A}} \W_1(p_{\Sa}, p_{\bS})\,.
\end{align*}
 We call this upper bound on SPDD \emph{pseudo-SPDD}. Pseudo-SPDD is the tightest upper bound to SPDD among all possible target distributions by the definition of the barycenter $p_{\bS}$ and Proposition~\ref{prop:1}. Indeed 
\begin{align*}
&\hskip-0.7cm \sum_{\va\in \mathcal{A}} \mathbb{E}_{\tau\sim U(\Omega)} |\mathbbm{P}(S_{\va}>\tau) - \mathbbm{P}(\bar S > \tau)| \\
& = \sum_{\va\in \mathcal{A}} \W_1(p_{\Sa}, p_{\bS}) \leq \sum_{\va\in \mathcal{A}} \W_1(p_{\Sa}, p_{S^0})\\
& = \sum_{\va\in \mathcal{A}} \mathbb{E}_{\tau\sim U(\Omega)} |\mathbbm{P}(S_{\va}>\tau) - \mathbbm{P}(S^0 > \tau)|,
\end{align*}
for any distribution $p_{S^0}\in\mathcal{P}(\Omega)$. Since SPDD is difficult to derive optimal trade-offs for, we do that with respect to the pseudo-SPDD as the measure of fairness instead.

We are interested in changing $p_{\Sa}$ to $p_{\Sa^*}$, $\forall \va \in\mathcal{A}$, to reach a fairness bound $\lambda\in \mathbb{R}_+$ for pseudo-SPDD such that the required model prediction changes are minimal in expectation. This is obtained by choosing the $p_{\Sa^*} \in\mathcal{P}(\Omega)$ that minimizes the total expected prediction changes, which equals $\sum_{\va\in\mathcal{A}}p_{\va}\W_1(p_{\Sa}, p_{\Sa^*})$ by Proposition~\ref{prop:1}, while bounding the pseudo-SPDD by $\lambda$, \ie~$\sum_{\va\in \mathcal{A}} \W_1(p_{\Sa^*}, p_{\bS}) \leq \lambda$. Assuming that the groups are disjoint, we can optimize each group transportation in turn independently assuming the other groups are fixed. This gives 
\begin{align*}
    p_{\Sa^*} &= \argmin_{\substack{p^*\in\mathcal{P}(\Omega) \, s.t.\\ \W_1(p^*, p_{\bar{S}}) \leq \lambda - \gamma}} p_{\va}\W_1(p_{\Sa}, p^*) \\
    &= \argmin_{\substack{p^*\in\mathcal{P}(\Omega) \, s.t.\\ \W_1(p^*, p_{\bS}) \leq \lambda - \gamma}} \W_1(p_{\Sa}, p^*)\,,
\end{align*}
where $\gamma = \sum_{\bar \va\in \mathcal{A}\setminus \va} \W_1(p_{S_{\bar{\va}}^*}, p_{\bS})$. 
By triangle inequality, $\W_1(p_{\Sa}, p^*) \geq |\W_1(p_{\Sa}, p_{\bS})-\W_1(p^*, p_{\bS})|$. 
The distance $\W_1(p_{\Sa}, p^*)$ reaches its minimum if and only if $p^*$ lies on a shortest path between $p_{\Sa}$ and $p_{\bS}$. Thus it is optimal to transport $p_{\Sa}$ along any shortest path between itself and $p_{\bS}$ in the Wasserstein-1 metric space. 
In the approach proposed in the next section, we approximate transporting group distributions along these shortest paths with hyperparameter tuning of a gradient descent method to minimize $\W_1(p_{\Sa}, p_{\bS})$ for every group.

\paragraph{Empirical Computation of the Barycenter.}
In practice, as building the barycenter from the population distributions $p_{S_a}$ is impossible, we use the empirical distributions $\hat{p}_{S_a}$ obtained from $\mathcal{D}_{\va}$. The choice is justified by the following result:
\begin{lemma}
If the samples in $\mathcal{D}$ are i.i.d., as $| \mathcal{D} | \rightarrow \infty$,  if $\W_1(p_{S}, p_{S_{\va}})< \infty $ for all $\va$, the empirical barycenter distribution satisfies  $\lim \sum_{\va} \hat{p}_{\va} \W_1(\hat{p}_{\bS}, \hat{p}_{S_\va}) \rightarrow \sum_{\va} p_{\va} \W_1(p_{\bS}, p_{S_\va})$ almost surely\footnote{See \cite{klenke13probability} for a formal definition of almost sure convergence of random variables.}. 
\end{lemma}
The proof is given in Appendix A.

In the next two sections we introduce two different approaches to achieve SDP with Wasserstein-1 distances: A penalization approach to logistic regression and a simpler practical approach consisting in post-processing model beliefs. 

\subsection{WASSERSTEIN-1 PENALIZED LOGISTIC REGRESSION\label{sec:wass1penalty}}
The average logistic regression loss function over $\mathcal{D}=\{\va^n,\vx^n,y^n\}_{n=1}^N$ is given by
\begin{equation*}
    J_{\mathcal{D}}(\vtheta) = \frac{1}{N}\sum_{n=1}^N -y^n \log s^n - (1-y^n)\log\left(1-s^n\right)\,,
\end{equation*}
where the model belief that individual $n$ belongs to class 1, $s^n$, is obtained as $s^n = \sigma(\vtheta^\top \vw^n) = 1/(1+ e^{-\vtheta^\top \vw^n})$, with $\vw^n = (\vx^n,\va^n,1)^\top$, and where $\vtheta \in \mathbb{R}^{d+k+1}$ are the model parameters. We denote with $\{ s_{\va}^i \}$ the model beliefs for group $\va$ and with $\{ \bar{s}^j\}_{j=1}^{\bar{N}}$ the atoms of $p_{\bS}$.

The gradient of $J_{\mathcal{D}}(\vtheta)$ with respect to $\vtheta$ is given by
\begin{equation*}
    \nabla_\vtheta J_\mathcal{D}(\vtheta)  = \frac{1}{N} \sum_{n=1}^N \vw^n\left(\sigma\left(\vtheta^\top\vw^n  \right) - y^n \right)\,.
\end{equation*}
We propose to find model parameters $\vtheta^*$ that minimize the population level logistic loss $\mathbb{E}\left[ J_\mathcal{D}(\vtheta)\right]$ under the constraint 
of small Wasserstein-1 distances $\W_1(\hat{p}_{S_\va}, \hat{p}_{\bS})$ between $\hat{p}_{S_\va}$ and the empirical barycenter $\hat{p}_{\bS}$, $\forall \va \in \mathcal{A}$. 

The Wasserstein-1 distance between any two empirical distributions $\hat{p}_{b}$ and $\hat{p}_{c}$ underlying two datasets $\{b^i \}_{i=1}^{N_b}, \{ c^j  \}_{j=1}^{N_c}  \subset \mathbb{R}$ is given by
\begin{equation}\label{equation::empirical_wass}
    \W_1(\hat{p}_b, \hat{p}_c) = \min_{T_{b,c}\in U(b,c)} \langle  T_{b,c} , \tC  \rangle\,,
\end{equation}
where $U(b,c) = \{ T \in \mathbb{R}^{N_b \times N_c}~\st~T_{b,c}\mathbf{1}_c = \frac{1}{N_b}\mathbf{1}_b \text{ and } T_{b,c}^\top \mathbf{1}_b = \frac{1}{N_c} \mathbf{1}_c\}$ with $\mathbf{1}_c$ denoting a vector of ones of size $N_c$.
The brackets $\langle \cdot, \cdot \rangle$ denote the trace dot product and $\tC$ is the cost matrix associated with the Wasserstein-1 cost function $\tc$ of elements $\tC_{i,j} = |b^i - c^j|$.

In particular, the Wasserstein-1 distance $\W_1(\hat{p}_{S_\va},\hat{p}_{\bS})$ can be computed by solving the optimization problem of \eqref{equation::empirical_wass} with cost matrix $\tC_{\va}^\vtheta \in \mathbb{R}^{N_{\va} \times \bar{N}}$ satisfying 
\begin{align*}
(\tC_{\va}^\vtheta)_{i,j} &= \left|s_{\va}^i - \bs^j  \right|\,,
\end{align*}
where the upper script $\vtheta$ in $\tC_{\va}^\vtheta$ is maintained to remind the reader that model predictions are a function of the model parameter $\vtheta$.

The Wasserstein-1 penalized logistic regression objective is given by
\begin{equation}\label{equation::wass_regularized_loss_objective}
J_{\W_1}(\vtheta) = \alpha J_\mathcal{D}(\vtheta) + (1-\alpha)\beta \sum_{\va \in \mathcal{A} }\W_1( \hat{p}_{S_\va}, \hat{p}_{\bS})\,,
\end{equation}
where $\alpha$ and $\beta$ are penalization coefficients.
\begin{lemma}
If the datasets $\{ b^i \}_{i=1}^{N_b}, \{ c^j \}_{j=1}^{N_c} \subset \mathbb{R}$ have empirical distributions $\hat{p}_b$ and $\hat{p}_c$, and $\tC$ is the cost matrix of elements $\tC_{i,j} = |b^i - c^j|$:
\begin{equation*}
    \nabla_{\tC} \hskip0.01cm \W_1(\hat{p}_b, \hat{p}_c)  = T_{b,c}^*\,,
\end{equation*}
where $T_{b,c}^* = \arg\min_{T_{b,c} \in U(b,c)} \langle T_{b,c}, \tC \rangle$ is the optimal coupling resulting from the optimization objective of \eqref{equation::empirical_wass}.
\end{lemma}
\begin{proof}
The result follows immediately from the subgradient rule for a pointwise max function (see \cite{boyd04convex}).
\end{proof}
\begin{lemma}\label{lemma::subgradient_full_cost}
The gradient of $J_{\W_1}(\vtheta)$ equals:
\begin{align*}
    \alpha \nabla_\vtheta &J_\mathcal{D}(\vtheta) + 
    (1-\alpha)\beta\big(\sum_{\va \in \mathcal{A}}\sum_{i,j} T_{\va}^*(\vtheta)\nabla_\vtheta \left| s_\va^i - \bar{s}^i  \right|   \big),
\end{align*}
where $T_{\va}^*$ is the optimal coupling between $\hat{p}_{S_\va}$ and $\hat{p}_{\bS}$\footnote{Recall that $T_{\va}^*$ is a function of $\vtheta$.}. 
\end{lemma}
\begin{proof}
This formula is a consequence of the chain rule and Lemma \ref{lemma::subgradient_cost}. 
\end{proof}

\begin{algorithm}[t]
\textbf{Input: Dataset $\mathcal{D} = \{ (\va^n, \vx^n, y^n)\}_{n=1}^N$, penalization coefficients $\alpha,\beta$, gradient step size $\eta$, number of optimization rounds $M$, frequency of barycenter computation $K$.}\\
Compute datasets $\{\mathcal{D}_{\va}\}$.\newline
Initialize model parameters $\vtheta_0$.\\
\For{$m = 1, \cdots , M$ }{
   1. Compute the barycenter distribution $\hat{p}_{\bS}$ (\cite{flamary2017pot}) once every $K$ steps, and $\{\bar s^i\}_{i=1}^{\bar{N}}$.\\
   2. Compute optimal couplings $\{ T_{\va^*}\}$ as defined in Lemma \ref{lemma::subgradient_full_cost}. \\
   3. Update parameter $\vtheta_m = \vtheta_{m-1} -\eta \nabla_\vtheta J_{\W_1}(\vtheta_m)$.
}
\textbf{Return:} $\vtheta_M$.
 \caption{Wass-1 Penalized Logistic Regression}
\label{Alg:wass1_penalty}
\end{algorithm}

\paragraph{Computation Method.}
We propose to optimize the Wasserstein penalized logistic loss objective (\eqref{equation::wass_regularized_loss_objective}) via gradient descent. The procedure is detailed in Algorithm \ref{Alg:wass1_penalty}. We start by describing how to perform Step 2. under the assumption that $\hat{p}_{\bS}$ and $\{\bar s^i\}_{i=1}^{\bar{N}}$ have been computed. The computation of the optimal coupling family $\{ T_\va^*\}$ hinges on the following Lemma.

\begin{lemma}\label{lemma::coupling_matrix}
If $\{ b^i \}_{i=1}^{N_b}, \{ c^j \}_{j=1}^{N_c} \subset \mathbb{R}$,  and $B_i = [i\times (N_c-1) + 1, \cdots , i\times N_c]$ for all $i$ and $C_j = [j\times (N_b-1) +1,\cdots, j\times N_c]$ for all $j$, then: 
    $ (T_{b,c}^*)_{i,j} = 
                    \frac{ \#\left| B_i \cap C_j \right|}{N_b N_c}$.
\end{lemma}
This lemma characterizes the coupling matrix between the empirical distributions of two datasets made of real numbers. When $N_b = N_c$ and the datasets are $\{b^i\}^N_b$ and $\{c^i\}^N_c$, with $b^1 < \ldots < b^N$, and $a^1 < \ldots < a^N$, then the optimal coupling equals $1/N \times I_N$ where $I_N$ denotes the $N \times N$ identity matrix. Lemma \ref{lemma::coupling_matrix} extends this simple case to the general case of datasets of arbitrary orderings and sizes, see \cite{deshpande18generative} for a proof. It is easy to see that the optimal coupling $T_{b,c}^*$ is sparse and has at most ${\cal O}(N_b + N_c)$ nonzero entries (see \cite{cuturi13sinkhorn}). As a consequence, the computation of $\nabla_\vtheta J_{\W_1}(\vtheta)$ can be performed in linear time ${\cal O}\big(\sum_{\va} (N_\va + \bar{N})\big)$ where $N_{\va}=| \mathcal{D}_{\va}|$. In the computation of $\nabla_\vtheta J_{\W_1}(\vtheta)$ only the nonzero entries of $T_{b,c}^*$ matter.

We compute the empirical barycenter $\hat{p}_{\bS}$ and $\{\bar s^i \}_{i=1}^{\bar{N}}$, using the POT library by \cite{flamary2017pot}. We fix the support of potential barycenters to bins of equal-width spanning the $[0,1]$ interval, and use the iterative KL-projection method proposed by \cite{benamou15iterative}. We then generate a number of samples from the normalized probability distribution of the computed barycenter. 

\paragraph{Demographically-Blind Wasserstein-1 Penalized Logistic Regression.}
In real-world applications, the use of sensitive attributes might be prohibited when deploying a system. We therefore consider the variation where $\vw^n = (\vx^n, 1)^\top$. This variation still uses the sensitive attributes to calculate the Wasserstein-1 loss but, by not including them into the feature set, does not require knowledge of sensitive information at test time. 

\subsection{WASSERSTEIN-1 POST-PROCESSING\label{sec:wass1post}}
In this section, we propose a simple, fast quantile matching method to post-process the beliefs of a classifier trained on $\mathcal{D}$. This method corresponds to an approximate Wasserstein-1 optimal transport map by the formulation of \citet{rachev98mass}:
\[
\W_1(p_{\Sa}, p_{\bS}) = \int_{\tau=0}^1 |P_{\Sa}^{-1}(\tau) - P_{\bS}^{-1}(\tau)|d\tau\,.
\]
The procedure is detailed in Algorithm~\ref{Alg:wass1_post_process}.
For each group $\va$, we compute quantiles of $\hat p_{\Sa}$ and map all group beliefs belonging in each quantile bin to the supremum of those belonging to the corresponding quantile bin of $\hat p_{\bS}$.

\subsection{GENERALIZATION}
The following lemma addresses generalization of the Wasserstein-1 objective. Assume $\W_1(p_{S_\va}, p_{\bS}) \leq L$ for all $\va\in \mathcal{A}$. Let $P_S, P_{S_\va}$ and $P_{\bS}$ be the cumulative density functions of $S$, $S_\va$ and $\bS$. Assume these random variables all have domain $\Omega = [0,1]$ and that all $P \in \{  P_S, P_{\bS}\} \cup \{  P_{S_\va}\}_{\va \in \mathcal{A}}$ are continuous, then:
\begin{lemma}\label{lemma:wass_gen}
For any $\epsilon, \delta > 0$, if 
$\min\big[\bar{N}, \min_{\va}\big[N_\va\big]\big]\geq  \frac{16\log(2|\mathcal{A}|/\delta)|\mathcal{A}|^2\max[1, L]^2}{\epsilon^2}$, with probability $1-\delta$:
\begin{equation*}
    \sum_{\va \in \mathcal{A}} p_{\va} \W_1(p_{S_{\va}}, p_{\bS})  
    \leq \sum_{\va \in \mathcal{A}} \hat{p}_{\va} \W_1(\hat{p}_{S_\va},\hat{p}_{\bS}) + \epsilon\,.
\end{equation*}
In other words, provided access to sufficient samples, a low value of $\sum_{\va} \hat{p}_{\va} \W_1(\hat{p}_{S_\va}, \hat{p}_{\bS})$ implies a low value for $\sum_{\va} p_{\va} \W_1(p_{S_\va}, p_{\bS}) $ with high probability and therefore good performance at test time. 
\end{lemma}
The proof is given in Appendix B.

Lemma \ref{lemma:wass_gen} implies that under appropriate conditions, the value of the population objective of the Wasserstein cost is upper bounded by  the empirical Wasserstein cost plus a small constant.
\begin{algorithm}[t]
\textbf{Input: dataset $\mathcal{D} = \{ (\va^n, \vx^n, y^n)\}_{n=1}^N$, set of quantile bins $\mathcal{B}$, model beliefs $\{s^n\}$} \\
Compute datasets $\{\mathcal{D}_\va\}$ and their barycenter $\bar{\mathcal{D}}$.
\newline
Define the $i$-th quantile of dataset $\mathcal{D}_\va$ as
$$q_{\mathcal{D}_{\va}}(i) := \sup \left\{s : \frac{1}{N_{\va}}\sum_{n ~\st~ \va^n=\va} \1_{s^n \le s} \le \frac{i-1}{|\mathcal{B}|}\right\},$$
and its inverse as $q^{-1}_{\mathcal{D}_{\va}}(s) := \sup\{ i \in \mathcal{B} : q_{ \mathcal{D}_{\va}}(i) \le s\}$. 
\newline 
\textbf{Return:}  $\Big\{ q_{\bar{\mathcal{D}}}\left(q^{-1}_{\mathcal{D}_{\va}}(s^n)\right)\Big\}$.\newline 
\caption{Wass-1 Post-Processing}
\label{Alg:wass1_post_process}
\end{algorithm}
\section{RELATED WORK}
\vskip-0.2cm
\begin{table*}
\setlength{\tabcolsep}{5pt}
\caption{{\bf Adult Dataset} -- {\bf German Credit Dataset}}
\begin{center}
\begin{tabular}{lccccc|ccccc}
&\multicolumn{5}{c}{Adult}&\multicolumn{5}{c}{German}\\
\cline{2-6}\cline{7-11}
& Err-.5 & Err-Exp & DD-.5 & SDD & SPDD & Err-.5 & Err-Exp & DD-.5 & SDD  & SPDD\\
\toprule
Unconstrained        & .142 & .198 & .413 & .426 & .806 &.248 & .319 & .124 & .102 & .103\\
Hardt's Post-Process & .165 & \boldred{.289} & .327 & .551 &1.058 & .248 & \boldred{.333} & .056 & .045 & .045\\
Constrained Opt.     & .205 & .198 & .065 & .087 & .166 & .318 & .320 & .173 & .149 & .149\\
Adv. Constr. Opt.    & .219 & .207 & .0   & .114 & .203 & .306 & .307 & .0   & .021 & .021\\
Wass-1 Penalty       & .199 & .208 & .014 & .022 & .044 & .306 & .311 & .0   & \boldblue{.003} & \boldblue{.003}\\
Wass-1 Penalty DB    & .230 & .233 & .010 & \boldblue{.012} & \boldblue{.023} & .306 & .309 & .0   & .010 & .010\\
Wass-1 Post-Process  &.174 & .214 & .013 & .017 & .042 & .258 & .327 & .068 & .023 & .023\\
Wass-1 Post-Process $\hat p_{S}$  & .165 & .216 & .032 & .022 & .059 & .248 & .320 & .056 & .025 & .025\\
\bottomrule
\end{tabular}
\end{center}
\label{table:Adult&German}
\end{table*}
Broadly speaking, we can group current literature on fair classification and regression into three main approaches. The first approach consists in pre-processing the data to remove bias, or in extracting representations that do not contain sensitive information during training \citep{beutel17data,calders09building,calmon17optimized,edwards16censoring,feldman15certifying,fish15fair,kamiran09classifying,kamiran12data,louizos16fair,zemel13learning,vzliobaite11handling}. 
This approach includes current methods to fairness using Wasserstein distances consisting in achieving SDP through transportation of features \citep{delbarrio18obtaining,johndrow19algorithm}.
The second approach consists in performing a post-processing of the model outputs \citep{chiappa19path,doherty12information,feldman15computational,hardt16equality,kusner17counterfactual}. 
The third approach consists in enforcing fairness notions by imposing constraints into the optimization, or by using an adversary. Some methods transform the constrained optimization problem via the method of Lagrange multipliers \citep{goh16satisfying,zafar17fairness,wu18preventing,agarwal18reduction,cotter18two,corbett-davies17algorithmic,narasimhan18learning}.
Other work similar in spirit adds penalties to the objective \citep{komiyama18nonconvex, donini18empirical}. Adversarial methods maximize the system ability to predict
$Y$ while minimizing the ability to predict $\vA$ \citep{zhang18mitigating}. 

\section{EXPERIMENTS}\label{sec:experiments}
In this section, we evaluate the methods introduced in Sections \ref{sec:wass1penalty} and \ref{sec:wass1post} on four datasets from the UCI repository \citep{lichman13uci}. For penalized logistic regression, we refer to the method in which sensitive information is included in the feature set, \ie~$\vw^n=(\vx^n, \va^n, 1)^\top$, as {\bf Wass-1 Penalty}; and to the demographically-blind variant in which sensitive information is not included, \ie~$\vw^n=(\vx^n, 1)^\top$, as {\bf Wass-1 Penalty DB}. We refer to the post-processing method as {\bf Wass-1 Post-Process}.
We also include a variant of this method using $\hat p_S$ instead of the barycenter $\hat p_{\bS}$ ({\bf Wass-1 Post-Process $\hat p_S$}), which gives a simpler algorithm that only requires computing basic quantile functions. We compare these methods with the following baselines:
\begin{description}[leftmargin=*]
\setlength{\itemsep}{-6pt}  
\setlength{\parskip}{-4pt}
\setlength{\parsep}{-4pt}
\item[{\bf Unconstrained}:] Logistic regression with no fairness constraints.\\
\item[{\bf Hardt's Post-Process}:] Post-processing of the logistic regression beliefs $s^n$ of all individuals in group $\va$ by adding $0.5 - \tau_{\va}$, where the threshold $\tau_{\va}$ is found using the method of \citet{hardt16equality}. This ensures that DP is satisfied at threshold $\tau=0.5$.\\
\item[{\bf Constrained Optimization:}] Lagrangian-based method (see \eg~\cite{eban17scalable,goh16satisfying}) using a linear model as the underlying predictor and equal positive prediction rate between each group $\mathcal{D}_{\va}$ and $\mathcal{D}$ as fairness constraints with threshold $\tau=0$.\\
\item[{\bf Adv. Constr. Opt.:}] The same as the previous method, but with more fairness constraints. Specifically, the fairness constraints are equal positive prediction rates for a set of thresholds from $-2$ to $2$ in increments of $0.2$ on the output of the linear model.
\end{description}

\begin{table*}
\setlength{\tabcolsep}{5pt}
\caption{{\bf Bank Marketing Dataset} -- {\bf Community \& Crime Dataset}}
\begin{center}
\begin{tabular}{lccccc|ccccc}
&\multicolumn{5}{c}{Bank Marketing}&\multicolumn{5}{c}{Community \& Crime}\\
\cline{2-6}\cline{7-11}
& Err-.5 & Err-Exp & DD-.5 & SDD & SPDD & Err-.5 & Err-Exp & DD-.5 & SDD  & SPDD\\
\toprule
Unconstrained        & .094 & .138 & .135 & .134  & .61 & .116 & .195 & .581 &1.402 &7.649\\
Hardt's Post-Process & .097 & \boldred{.181} & .018 & .367  &1.057& .321 & \boldred{.441} & .226 & .536 &2.679\\
Constrained Opt.     & .105 & .110 & .049 & .026  & .076& .289 & .263 & .193 & .369 &2.003\\
Adv. Constr. Opt.    & .105 & .105 & .050 & .064  & .184& .303 & .275 & .022 & .312 &1.628\\
Wass-1 Penalty       & .114 & .151 & .001 &.015 &.050 & .313 & .315 & .0   & \boldblue{.008} & \boldblue{.039}\\
Wass-1 Penalty DB    & .114 & .131 & .001 &\boldblue{.006}&\boldblue{.018}& .313 & .315 & .0   & .011 & .051\\
Wass-1 Post-Process  & .100 & .144 & .016 & .020  & .062& .321 & .363 & .226 & .133 & .680\\
Wass-1 Post-Process $\hat p_{S}$  & .097 & .141 & .014 & .020  & .063& .321 & .335 & .226 & .159 & .822\\
\bottomrule
\end{tabular}
\end{center}
\label{table:Bank&Crime}
\end{table*}

\subsection{TRAINING DETAILS} 
In the approaches Unconstrained, Hardt's Post-Process, Wass-1 Penalty, and Wass-1 Post-Process, we trained a logistic regression model using Scikit-Learn with default hyper-parameters \citep{pedregosa11scikit-learn}. 

For Wass-1 Penalty (Algorithm~\ref{Alg:wass1_penalty}), as initial model parameters $\vtheta_0$ we used the ones given by the trained logistic regression. We swept over penalization coefficients $\alpha=[0, 0.5]$, $\beta=[10^{-2}, 3\cdot 10^{-2}, 10^{-1}, 3\cdot 10^{-1}, 1, 3, 10, 30, 10^2]$, gradient step sizes $\eta=[10^{-4}, 10^{-3}, 10^{-2}, 10^{-1}]$, set the maximum number of training steps to $M=80,000$, and computed the barycenter once every $K>M$ steps, effectively only once after the initialization of $\vtheta_0$.
In the computation of the barycenter (using the POT library by \cite{flamary2017pot}), we swept over numbers of bins $B = [50, 90]$, entropy penalty $\delta = [10^{-3}, 5\cdot 10^{-3}, 10^{-2}]$, and used number of iterations $M=1,000$. The time complexity of our implementation is $\mathcal{O}(N \log(N))$. Our gradient steps take on average $\sim$0.02 seconds. 

For Wass-1 Post-Process (Algorithm~\ref{Alg:wass1_post_process}), we used a number of bins $|{\cal B}|= 100$.

For Constrained Optimization, we used the hinge loss as objective and the hinge relaxation for the fairness constraints. We trained by jointly optimizing the model parameters and Lagrange multipliers on the Lagrangian using ADAM with the default step-size of $0.001$ and mini-batch size of $100$, and trained for $50$ steps. We allowed an additive slack of $0.05$ on the constraints, as otherwise we found feasibility issues leading to degenerate classifiers.

\subsection{DATASETS} 
{\bf The UCI Adult Dataset.}
The Adult dataset contains 14 attributes including age, working class, education level, marital status, occupation, relationship, race, gender, capital gain and loss, working hours, and nationality for 48,842 individuals; 32,561 and 16,281 for the training and test sets respectively. The goal is to predict whether the individual's annual income is above or below \$50,000. 

\emph{Pre-processing and Sensitive Attributes.} We pre-processed the data in the same way as done in \cite{zafar17fairness,goh16satisfying}. The categorical features were encoded into binary features (one for each category), and the continuous features were transformed into binary encodings depending on five quantile values, obtaining a total of $122$ features.
As sensitive attributes, we considered race (Black and White) and gender (female and male), obtaining four groups corresponding to black females, white females, black males, and white males.\\

{\bf The UCI German Credit Dataset.}
This dataset contains 20 attributes for 1,000 individuals applying for loans. Each applicant is classified as a good or bad credit risk, \ie~as likely or not likely to repay the loan. We randomly divided the dataset into training and test sets of sizes 670 and 330 respectively. 

\emph{Pre-processing and Sensitive Attributes.} We did not do any pre-processing. As sensitive attributes, we considered age ($\leq 30$ and $>30$ years old), obtaining two groups.\\

{\bf The UCI Bank Marketing Dataset.}
This dataset contains 20 attributes for 41,188 individuals. 
Each individual is classified as subscribed or not to a term deposit. We divided the dataset into train and test sets of sizes 32,950 and 8,238 respectively.

\emph{Pre-processing and Sensitive Attributes.} We pre-processed the data as for the Adult dataset. We transformed the categorical features into binary ones, and the continuous features into five binary features based on five quantile bins, obtaining a total of 60 features. We also subtracted the mean from cons.price.idx, cons.conf.idx, euribor3m, and nr.employed to make them zero-centered. As sensitive attributes, we considered age, which was discretized based on five quantiles leading to five groups.

{\bf The UCI Communities \& Crime Dataset.}
This dataset contains 135 attributes for 1994 communities; 1495 and 499 for the training and test sets respectively. The goal is to predict whether a community has high (above the 70-th percentile) crime rate.

\emph{Pre-processing and Sensitive Attributes.} We pre-processed the data as in \cite{wu18preventing}. As sensitive attributes, we considered race (Black, White, Asian and Hispanic), thresholded at the median to form height groups. 

\subsection{RESULTS}
\begin{figure*}
\subfigure{
\hspace{-2.6cm}\includegraphics[height=4cm,width=21cm]{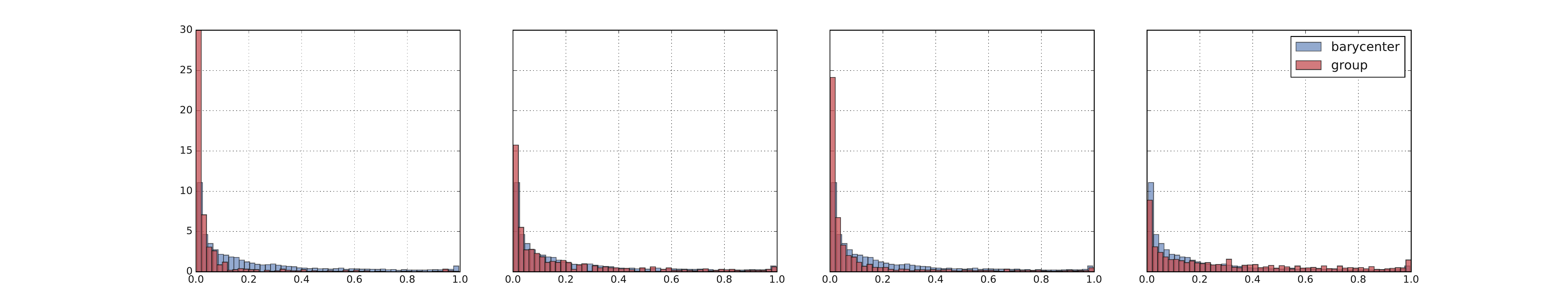}}
\subfigure{
\hspace{-2.6cm}\includegraphics[height=4cm,width=21cm]{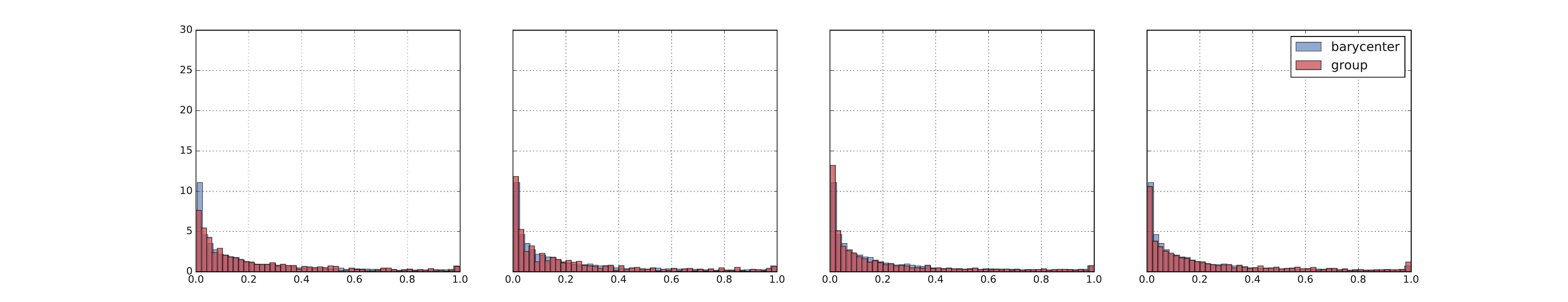}}
\caption{Histograms of model beliefs for groups of Black females, Black males, White females, and White males, and their barycenter on the Adult dataset using Wass-1 Penalty. Top: Initial state. 
Bottom: After 10,000 training steps with $\alpha=0,\beta=100$ each group histogram matches the barycenter.}
\label{fig:hists}
\end{figure*}

We compared the different methods using the following metrics:
\begin{description}
\setlength{\itemsep}{-6pt}  
\setlength{\parskip}{-4pt}
\setlength{\parsep}{-4pt}
\item[{\bf Err-.5:}] Binary classification error using threshold $\tau=0.5$, \ie~$\textrm{Err-.5} =\frac{1}{N}\sum_{n=1}^N\mathbbm{1}_{\hat y^n \neq y^n}$.\\
\item[{\bf Err-Exp:}] As above, but averaging over 100 uniformly-spaced thresholds $\tau \in [0,1]$.\\ 
\item[{\bf DD-.5:}] Demographic disparity at $\tau = 0.5$, summed over all groups $\va\in \mathcal{A}$, \ie~$\textrm{DD-.5} = \sum_{\va \in \mathcal{A}} |\mathbbm{P}(\Sa > 0.5) - \mathbbm{P}(S > 0.5)|$, where \eg~$\mathbbm{P}(S > \tau)$ is estimated as $\mathbbm{P}(S > \tau)\approx \frac{1}{N}\sum_{n=1}^N \mathbbm{1}_{s^n > \tau}$.\\
\item[{\bf SDD}] (strong demographic disparity): As above, but averaging over 100 uniformly-spaced thresholds $\tau \in [0,1]$, \ie~$\textrm{SDD}=\sum_{\va \in \mathcal{A}}\mathbb{E}_{\tau\sim U([0,1])} |\mathbbm{P}(S_{\va}>\tau) - \mathbbm{P}(S > \tau)|$. 
We use this metric to compare with other baselines that use the full-dataset belief distribution.\\
\item[{\bf SPDD:}] 
$\textrm{SPDD}=\sum_{\va,\bar\va\in \mathcal{A}} \mathbb{E}_{\tau\sim U([0,1])} |\mathbbm{P}(S_{\va}>\tau) - \mathbbm{P}(S_{\bar\va} > \tau)|$.
This metric is the most important, target-neutral, (un)fairness measurement as it does not depend on the target distribution, \eg~the full-dataset belief distribution or the barycenter.  
\end{description}

Figure~\ref{fig:hists} shows overlaying model belief histograms for four demographic groups and their barycenter in the Adult dataset. Wasserstein-1 Penalty  effectively matches  all group histograms to the barycenter after training for 10,000 steps with $\beta=100$. 

The main experiment results are shown in Tables~\ref{table:Adult&German} and \ref{table:Bank&Crime}\footnote{Given the deterministic baseline logistic regression model, all standard deviations are on the order of $10^{-4}$ or below.}. 
Focusing on the three more relevant metrics -- namely Err-Exp as the robust error measure, SDD as the conventional fairness comparison metric, and SPDD as the target-neural, preferred fairness metric (according to which we picked the best hyperparameter settings) --  
we can see that Wass-1 Penalty and Wass-1 Penalty DB have lowest SDD and SPDD (blue) on the German and Crime datasets and on the Adult and Bank datasets respectively. The fairness performance of these two methods are followed closely by the simpler Wass-1 Post-Process methods on all datasets. Hardt's Post-Process method incurs largest errors (red) on all datasets. After the Unconstrained baseline, Constrained Optimization and Adv. Contr. Opt. give lowest error on the Adult, Bank and Crime datasets, whilst Constrained Optimization and Wass-1 Penalty (DB) give lowest error on the German dataset. Overall the Wasserstein-1 methods gave best fairness performance on all the datasets with similar or lower compromise on accuracy than the baselines.

\begin{figure}[t]
\centering
\includegraphics[height=2.1in, trim={2cm 0 2cm 1cm},clip]{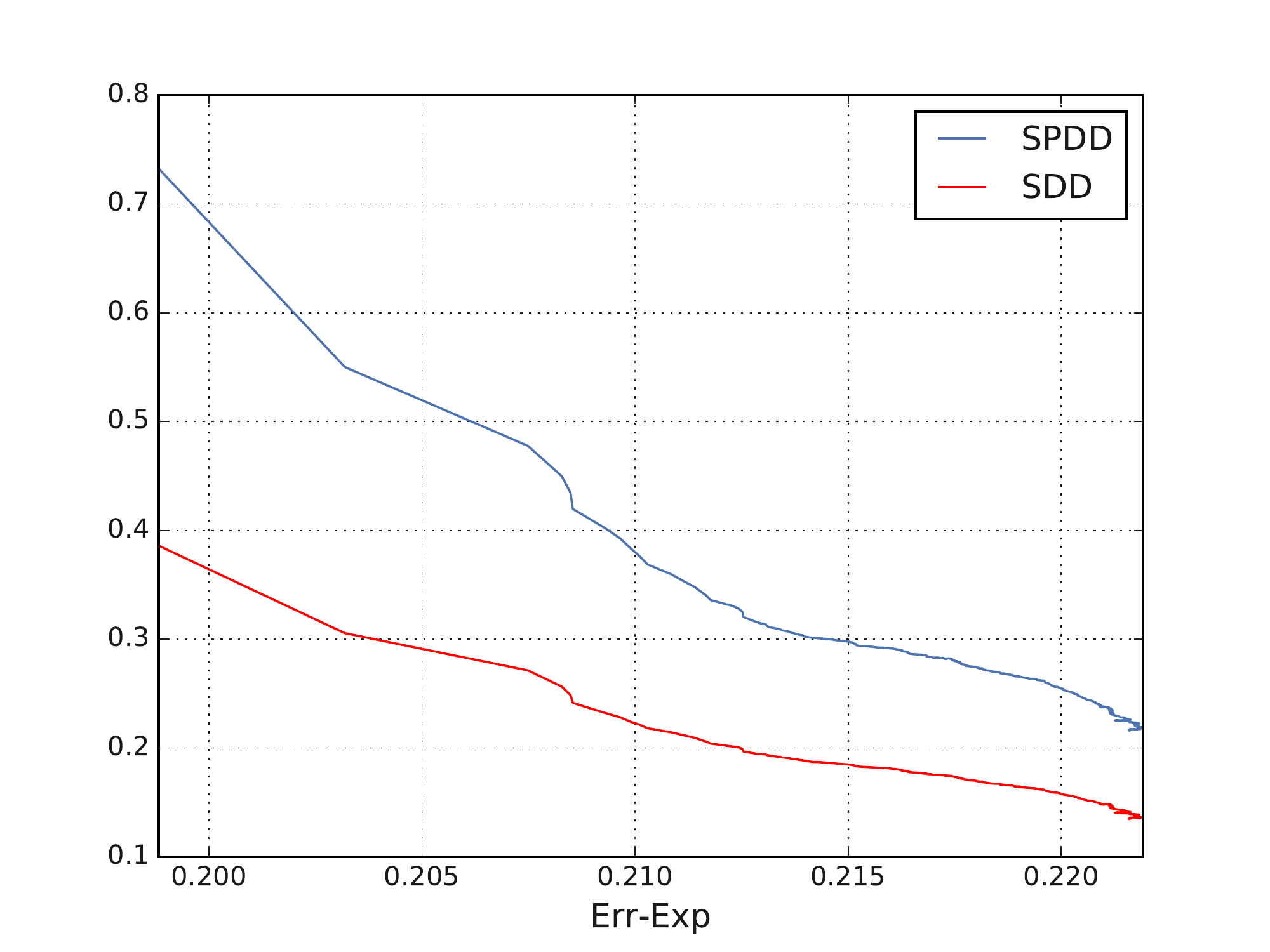}
\vskip-0.2cm
\caption[]{Err-Exp \versus~SDD, Err-Exp \versus~SPDD trade-off curves on Bank test set using Wass-1 Penalty DB, points plotted every 100 steps over 80,000 total training steps.}
\label{fig:tradeoffs}
\end{figure}
Since Wass-1 Penalty is trained by gradient descent, early-stopping can be an effective way to control trade-off between accuracy and fairness. Figure~\ref{fig:tradeoffs} shows a typical example of two trade-off curves between SDD/SPDD and Err-Exp. Though not always the case, often as the learning model moves towards the fairness goal of SDP, model accuracy decreases (Err-Exp increases). 

\section{CONCLUSIONS}
We introduced an approach to ensure that the output of a classification system does not depend on sensitive information using the Wasserstein-1 distance. 
We demonstrated that using the Wasserstein-1 barycenter enables us to reach independence with minimal modifications of the model decisions. 
We introduced two methods with different desirable properties, a Wasserstein-1 constrained method that does not necessarily require access to sensitive information at deployment time, and an alternative fast and practical approximation method that requires knowledge of sensitive information at test time. We showed that these methods outperform previous approaches in the literature.
\subsection*{Acknowledgements}
The authors would like to thank Mark Rowland for useful feedback on the manuscript.
\bibliography{main}
\bibliographystyle{plainnat}

\clearpage

\appendix
\onecolumn
{\Large \bf Appendix}
\section{Empirical Estimates}

\setcounter{lemma}{0}
\begin{lemma}\label{lemma::subgradient_cost}
As $| \mathcal{D} | \rightarrow \infty$, if $\W_1(p_{S}, p_{S_{\va}})< \infty $ for all $\va$, the empirical barycenter satisfies $\lim \sum_{\va} \hat{p}_{\va} \W_1(\hat{p}_{\bS}, \hat{p}_{S_\va}) \rightarrow \sum_{\va} p_{\va} \W_1(p_{\bS}, p_{S_\va})$ almost surely\footnote{See \cite{klenke13probability} for a formal definition of almost sure convergence of random variables.}. \end{lemma}
\begin{proof}
By triangle inequality:
\begin{align}
\sum_{\va} \hat{p}_{\va} \W_1(\hat{p}_{\bS}, p_{S_\va}) &\leq \sum_{\va} \hat{p}_{\va}\W_1(\hat{p}_{\bS}, \hat{p}_{S_\va}) +  \hat{p}_{\va}\W_1(p_{S_\va}, \hat{p}_{S_\va})\,, \label{eq::inequality_barycenter_1}\\
\sum_{\va} p_{\va}\W_1(p_{\bS}, \hat{p}_{S_\va}) &\leq \sum_{\va}p_{\va} \W_1(p_{\bS}, p_{S_\va}) +  p_{\va}\W_1(p_{S_\va}, \hat{p}_{S_\va})\,. \label{eq::inequality_barycenter_2}
\end{align}
Since $p_{\bar{S}}$ and $\hat{p}_{\bar{S}}$ are the weighted barycenters of $\{p_{S_{\va}}\}$ and $\{ \hat{p}_{S_{\va}}\}$ respectively:
\begin{align}
\sum_{\va} p_{\va}\W_1(p_{\bS}, p_{S_\va}) &\leq \sum_{\va}p_{\va} \W_1(\hat{p}_{\bS} , p_{S_\va}) \label{eq::barycenter_one} \,,\\
\sum_{\va}\hat{p}_{\va} \W_1(\hat{p}_{\bS}, \hat{p}_{S_\va} ) &\leq \sum_{\va} \hat{p}_{\va} \W_1(p_{\bS}, \hat{p}_{S_\va})\,. \label{eq::barycenter_two}
\end{align}
Combining Eqs. (\ref{eq::inequality_barycenter_1}) and (\ref{eq::barycenter_one}), and (\ref{eq::inequality_barycenter_2}) and (\ref{eq::barycenter_two}):

\begin{align*}
    \sum_{\va} p_{\va} \W_1(p_{\bS}, p_{S_{\va}} ) &\leq \sum_{\va} p_{\va} \W_1( \hat{p}_{\bS}, \hat{p}_{S_{\va}}) + p_{\va}\W_1(p_{S_{\va}}, \hat{p}_{S_{\va}}) \\
    &\leq \sum_{\va} \hat{p}_{\va} \W_1( \hat{p}_{\bS}, \hat{p}_{S_{\va}}) +   |  \hat{p}_{\va} \W_1( \hat{p}_{\bS}, \hat{p}_{S_{\va}}) - p_{\va} \W_1( \hat{p}_{\bS}, \hat{p}_{S_{\va}})    | + p_{\va}\W_1(p_{S_{\va}}, \hat{p}_{S_{\va}}) \\
    & \leq \sum_{\va} \hat{p}_{\va} \W_1( \hat{p}_{\bS}, \hat{p}_{S_{\va}}) +   |  \hat{p}_{\va} - p_{\va} |\cdot|   \W_1( \hat{p}_{\bS}, \hat{p}_{S_{\va}}) | + p_{\va}\W_1(p_{S_{\va}}, \hat{p}_{S_{\va}})\\
         \sum_{\va} \hat{p}_{{\va}} \W_1( \hat{p}_{\bS}, \hat{p}_{S_{\va}})&\leq   \sum_{\va} \hat{p}_{\va} \W_1(p_{\bS}, p_{S_{\va}} ) +\hat{p}_{\va} \W_1(p_{S_{\va}}, \hat{p}_{S_{\va}}) \\
         &\leq \sum_{\va} p_{\va} \W_1(p_{\bS}, p_{S_{\va}} ) + |p_{\va} \W_1(p_{\bS}, p_{S_{\va}} )-\hat{p}_{\va} \W_1(p_{\bS}, p_{S_{\va}} ) | +\hat{p}_{\va} \W_1(p_{S_{\va}}, \hat{p}_{S_{\va}}) \\
         &\leq \sum_{\va} p_{\va} \W_1(p_{\bS}, p_{S_{\va}} ) + | p_{\va} - \hat{p}_{\va}|\cdot| \W_1(p_{\bS}, p_{S_{\va}} ) | +\hat{p}_{\va} \W_1(p_{S_{\va}}, \hat{p}_{S_{\va}})\,.
\end{align*}
Therefore the following inequality holds almost surely:
\begin{align*}
    \Big|  \sum_{\va}  p_{\va}\W_1(p_{\bS}, p_{S_{\va}} )-   \sum_{\va} \hat{p}_{\va} \W_1( \hat{p}_{\bS}, \hat{p}_{S_{\va}}) \Big| &\leq\sum_{\va}  \hat{p}_{\va} \W_1(p_{S_{\va}}, \hat{p}_{S_{\va}}) + | p_{\va} - \hat{p}_{\va}|\cdot \W_1(p_{\bS}, p_{S_{\va}} )  \\
    &\leq \sum_{\va}  \W_1(p_{S_{\va}}, \hat{p}_{S_{\va}}) + | p_{\va} - \hat{p}_{\va}|\cdot \W_1(p_{\bS}, p_{S_{\va}} )  \\
    &\leq \sum_{\va}  \W_1(p_{S_{\va}}, \hat{p}_{S_{\va}}) + | p_{\va} - \hat{p}_{\va}|\cdot  \W_1(p_{S}, p_{S_{\va}} )\,.
\end{align*}


Since $\W_1(p_{S_\va}, \hat{p}_{S_{\va}})\rightarrow 0$ almost surely for all $\va$ (see \cite{weed17sharp}), and $\hat{p}_{\va} \rightarrow p_{{\va}}$ almost surely (by the strong law of large numbers) and $\W_1(p_{S}, p_{S_{\va}} ) < \infty$ for all $\va$, the result follows:

\begin{equation*}
\lim \sum_{\va} \hat{p}_{{\va}} \W_1(\hat{p}_{\bS}, \hat{p}_{S_\va}) \rightarrow \sum_{\va} p_{\va}\W_1(p_{\bS}, p_{S_\va})\,, 
\end{equation*}
almost surely.
\end{proof}

\section{Generalization}
The following lemma addresses generalization of the Wasserstein-1 objective. Assume $\W_1(p_{S_\va}, p_{\bS}) \leq L$ for all $\va\in \mathcal{A}$. Let $P_S, P_{S_\va}$ and $P_{\bS}$ be the cumulative density functions of $S$, $S_\va$ and $\bS$. Assume these random variables all have domain $\Omega = [0,1]$ and that all $P \in \{  P_S, P_{\bS}\} \cup \{  P_{S_\va}\}_{\va \in \mathcal{A}}$ are continuous, then:

\setcounter{lemma}{4}
\begin{lemma}\label{lemma:was_gen} 
For any $\epsilon, \delta > 0$, if 
$\min\big[\bar{N}, \min_{\va}\big[N_\va\big]\big]\geq  \frac{16\log(2|\mathcal{A}|/\delta)|\mathcal{A}|^2\max[1, L]^2}{\epsilon^2}$, with probability $1-\delta$:
\begin{equation*}
    \sum_{\va \in \mathcal{A}} p_{\va} \W_1(p_{S_{\va}}, p_{\bS})  
    \leq \sum_{\va \in \mathcal{A}} \hat{p}_{\va} \W_1(\hat{p}_{S_\va},\hat{p}_{\bS}) + \epsilon\,.
\end{equation*}
In other words, provided access to sufficient samples, a low value of $\sum_{\va} \hat{p}_{\va} \W_1(\hat{p}_{S_\va}, \hat{p}_{\bS})$ implies a low value for $\sum_{\va} p_{\va} \W_1(p_{S_\va}, p_{\bS}) $ with high probability and therefore good performance at test time. 
\end{lemma}

\begin{proof}
We start with the case when $p_{\bS} = p_{S}$.
By the triangle inequality for Wasserstein-1 distances, for all $\va \in\mathcal{A}$:
\begin{align}\label{equation::triangle_inequality}
\hskip-0.3cm  \hat{p}_{\va} \W_1( p_{S_\va} , p_{\bS} ) \leq \hat{p}_{\va} \W_1( \hat{p}_{S_\va}, \hat{p}_{\bS}   ) + \hat{p}_{\va} \W_1( \hat{p}_{\bS}, p_{\bS} ) +\hat{p}_{\va} \W_1(\hat{p}_{S_\va}, p_{S_\va} )\,.
\end{align}
Let $\hat{P}$ for $P \in \{  P_S, P_{\bS}\} \cup \{  P_{S_\va}\}_{\va \in \mathcal{A}}$ denote the empirical CDF of $P$. Since their domain is restricted to $[0,1]$ and are one dimensional random variables:
\begin{equation}
    \W_1(\hat{p}_{S_*}, p_{S_*}) = \int_{0}^1 |\hat{P}(x) - P(x) | dx\,.
\end{equation}
For $S_* \in \{  S, {\bS}\} \cup \{  {S_\va}\}_{\va \in \mathcal{A}}$. Since $P \in \{  P_S, P_{\bS}\} \cup \{  P_{S_\va}\}_{\va \in \mathcal{A}}$ are all continuous, the Dvorestky-Kiefer-Wolfowitz theorem (see main theorem in \cite{massart90tight} ) and the condition $\min\big[\bar{N}, \min_{\va}\big[N_\va\big]\big]\geq  \frac{16\log(2|\mathcal{A}|/\delta)|\mathcal{A}|^2\max[1, L]^2}{\epsilon^2}$ implies that:
\begin{equation*}
    \mathbb{P}\left( \sup_{x\in [0,1]} | \hat{P}(x) - P(x) | \geq \frac{\epsilon}{4} \right) \leq  \frac{\delta}{2 | \mathcal{A}| }\,.
\end{equation*}
Since all the random variables have domain $[0,1]$ this in turn implies that for all $S_* \in \{  S, {\bS}\} \cup \{  {S_\va}\}_{\va \in \mathcal{A}}$:
\begin{equation*}
    \mathbb{P}\left( \W_1(\hat{p}_{S_*}, p_{S_*}) \geq \frac{\epsilon}{4} \right) \leq  \frac{\delta}{2 | \mathcal{A}| }\,.
\end{equation*}
And therefore that with probability $\geq 1-\frac{\delta}{2}$ the following inequalities hold simultaneously for all $\va \in \mathcal{A}$:
\begin{equation}\label{eq::wass_bound_empirical1}
\hat{p}_{\va} \W_1(\hat{p}_{\bS}, p_{\bS} )   \leq \frac{\hat{p}_{\va} \epsilon}{4}, \hskip0.3cm
    \hat{p}_{\va} \W_1( \hat{p}_{S_\va},  p_{S_\va}) \leq \frac{\hat{p}_{\va} \epsilon}{4}\,.
\end{equation}
Summing \eqref{equation::triangle_inequality} over $\va$ and applying the last observation yields
\begin{equation*}
    \sum_{\va \in \mathcal{A}} \hat{p}_{\va}\W_1(p_{S_{\va}}, p_{\bS}) \leq \sum_{\va \in \mathcal{A}} \hat{p}_{\va} \W_1(\hat{p}_{S_{\va}}, \hat{p}_{\bS}) + \frac{\epsilon  }{2}\,.
\end{equation*}
Recall that we assume $\forall \va\in\mathcal{A}$,
\begin{equation*}
    \W_1(p_{S_\va}, p_{\bar{S}}) \leq L\,.
\end{equation*}
By concentration of measure of Bernoulli random variables, with probability $\geq 1-\frac{\delta}{2}$ the following inequality holds simultaneously for all $\va \in\mathcal{A}$: 
\begin{equation}\label{eq::wass_bound_empirical2}
    |p_\va - \hat{p}_{\va}| \leq \frac{\epsilon}{4| \mathcal{A} | \max[L, 1]}\,.
\end{equation}
Consequently the desired result holds:
\begin{equation*}
    \sum_{\va \in \mathcal{A}} p_{\va}\W_1(p_{S_{\va}}, p_{\bS}) \leq \sum_{\va \in \mathcal{A}} \hat{p}_{\va} \W_1(\hat{p}_{S_{\va}}, \hat{p}_{\bS}) + \epsilon\,.
\end{equation*}
If $p_{\bS}$ equals the weighted barycenter of the population level distributions $\{p_{S_a}\}$, then 
\begin{equation*}
\sum_{\va \in \mathcal{A}} p_{\va} \W_1(p_{S_a}, p_{\bS}) \leq \sum_{\va \in \mathcal{A}} p_{\va} \W_1(p_{S_a}, \hat{p}_{\bS})\,.
\end{equation*}
Since $\hat{p}_{\va} \W_1(p_{S_{\va}}, \hat{p}_{\bS}) \leq \hat{p}_{\va} \W_1(\hat{p}_{S_\va}, \hat{p}_{\bS}) + \hat{p}_{\va}\W_1(\hat{p}_{S_{\va}}, p_{S_a})$, with probability $1-\delta$:
\begin{align*}
    \sum_{\va \in \mathcal{A}} p_{\va} \W_1(p_{S_a}, p_{\bS}) &\leq  \sum_{\va \in \mathcal{A}} \hat{p}_{\va} \W_1(p_{S_a}, p_{\bS}) + \frac{\epsilon}{2}  \\
    &\leq \sum_{\va \in \mathcal{A}} \hat{p}_{\va} \W_1(\hat{p}_{S_\va}, \hat{p}_{\bS}) + \hat{p}_{\va}\W_1(\hat{p}_{S_{\va}}, p_{S_a}) + \frac{\epsilon}{2}\\
    &\leq \sum_{\va \in \mathcal{A}} \hat{p}_{\va} \W_1(\hat{p}_{S_\va}, \hat{p}_{\bS}) + \epsilon
\end{align*}

The first inequality follows from~\eqref{eq::wass_bound_empirical2}, and the third one by~\eqref{eq::wass_bound_empirical1}. The result follows.

\end{proof}

\section{Inverse CDFs}
\begin{lemma}
Given two differentiable and invertible cumulative distribution functions $f,g$ over the probability space $\Omega=[0,1]$, thus $f, g : [0,1] \rightarrow [0, 1]$, we have 
\begin{align}
    \int_{s=0}^1 |f^{-1}(s) - g^{-1}(s)|ds = \int_{\tau=0}^1 |f(\tau) - g(\tau)|d\tau.
    \label{eq:inverse_cdf}
\end{align}
\label{lemma::inverse_cdf}
\end{lemma}

\begin{figure}[t]
\subfigure[Left side of \eqref{eq:inverse_cdf}]{\hspace{1.5cm}
\includegraphics[height=2in,width=2.2in]{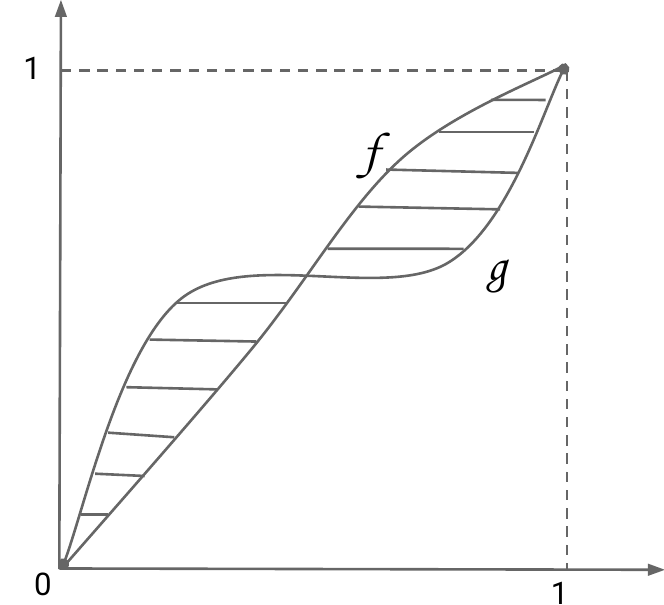}}
\subfigure[Right side of \eqref{eq:inverse_cdf}]{\hspace{1.5cm}
\includegraphics[height=2in,width=2.2in]{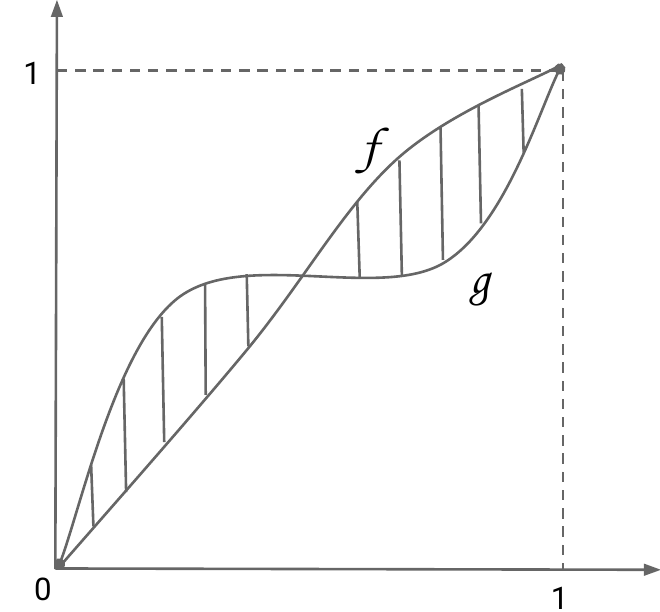}}
\caption[]{Integrating $|f^{-1} - g^{-1}|$ along the $x$ axis (left) and integrating $|f-g|$ along the $y$ axis (right) both compute the area of the same shaded region, thus the equality in \eqref{eq:inverse_cdf}.}
\label{fig:inverse_cdf}
\end{figure}
Intuitively, we see that the left and right side of \eqref{eq:inverse_cdf} correspond to two ways of computing the same shaded area in Figure \ref{fig:inverse_cdf}. Here is a complete proof.
\begin{proof}
 Invertible CDFs $f,g$ are strictly increasing functions due to being bijective and non-decreasing. Furthermore, we have $f(0)=0, f(1)=1$ by definition of CDFs and $\Omega=[0,1]$, since $P(X\leq 0) = 0, P(X\leq 1) = 1$ where $X$ is the corresponding random variable. The same holds for the function $g$. Given an interval $(x_1, x_2)\subset[0,1]$, let $y_1 = f(x_1), y_2 = f(x_2)$. Since $f$ is differentiable, we have
 \begin{align}
     \int_{x=x_1}^{x_2} f(x)dx + \int_{y=y_1}^{y_2} f^{-1}(y)dy = x_2 y_2- x_1 y_1.
     \label{eq:laisant}
 \end{align} 
 The proof of \eqref{eq:laisant} is the following (see also \cite{laisant05integration}). \begin{align*}
     &f^{-1}(f(x)) = x\\
     \Longrightarrow &f'(x) f^{-1}(f(x)) = f'(x)x \tag*{(multiply both sides by $f'(x)$)}\\
     \Longrightarrow &\int_{x=x_1}^{x_2} f'(x) f^{-1}(f(x)) dx= \int_{x=x_1}^{x_2} f'(x)x dx \tag*{(integrate both sides)}\\
     \Longrightarrow &\int_{y=y_1}^{y_2} f^{-1}(y)dy = \int_{x=x_1}^{x_2} f'(x)x dx \tag*{(apply change of variable $y=f(x)$ on the left side)}\\
     \Longrightarrow &\int_{y=y_1}^{y_2} f^{-1}(y)dy = xf(x){\bigg|}_{x=x_1}^{x_2} - \int_{x=x_1}^{x_2} f(x) dx \tag*{(integrate by parts on the right side)}\\
     \Longrightarrow &\int_{y=y_1}^{y_2} f^{-1}(y)dy +\int_{x=x_1}^{x_2} f(x)dx = x_2 y_2- x_1 y_1.
 \end{align*} 
Define a function $h \vcentcolon= f-g$ on $[0,1]$. Then $h$ is differentiable and thus continuous. Define the set of roots $A \vcentcolon= \{x \in [0,1] \mid h(x) = 0\}$. Define the set of open intervals on which either $h>0$ or $h<0$ by $B \vcentcolon= \{(a,b) \mid b = \inf\{s\in A\mid a<s\}, 0\leq a<b\leq 1, a\in A\}$. By continuity of $h$, for any $(a,b)\in B$, we have $b\in A$, \ie~$b$ is also a root of $h$. Since there are no other roots of $h$ in $(a,b)$, by continuity of $h$, we must have either $h>0$ or $h<0$ on $(a,b)$. For any two elements $(a,b), (c,d)\in B$, we argue that they must be disjoint intervals. Without loss of generality, we assume $a<c$. Since $b=\inf\{s\in A\mid a<s\}\leq c$, \ie~$b\leq c$, then $(a,b)\cap(c,d)=\emptyset$. For any open interval $(a,b)\in B$, there exists a rational number $q\in\mathbb{Q}$ such that $a<q<b$. We pick such a rational number and call it $q_{(a,b)}$. Since all elements of $B$ are disjoint, for any two intervals $(a_0, b_0), (a_1, b_1)$ containing $q_{(a_0, b_0)}, q_{(a_1, b_1)}\in\mathbb{Q}$ respectively, we must have $q_{(a_0, b_0)}\neq q_{(a_1, b_1)}$. We define the set $Q_B \vcentcolon= \{q_{(a,b)}\in\mathbb{Q} \mid (a,b)\in B\}$. Then $Q_B\subset \mathbb{Q}$ and $|Q_B| = |B|$. Since the set of rational numbers $\mathbb{Q}$ is countable, the set $B$ must also be countable. Let $B=\{(a_i, b_i)\}_{i=0}^N$ where $N\in\mathbb{N}$ or $N=\infty$. Recall that $h = f-g$ on $[0,1]$, $h(a_i)=0, h(b_i)=0$ and either $h<0$ or $h>0$ on $(a_i, b_i)$ for $\forall i>0$.

Consider the interval $(a_i, b_i)$ for some $i>0$, by Eq.\ref{eq:laisant} we have 
\begin{align*}
    &\int_{\tau=a_i}^{b_i} f(\tau)d\tau + \int_{s=f(a_i)}^{f(b_i)} f^{-1}(s)ds = b_i f(b_i)- a_i f(a_i) \\
    &= b_i g(b_i)- a_i g(a_i) = \int_{\tau=a_i}^{b_i} g(\tau)d\tau + \int_{s=g(a_i)}^{g(b_i)} g^{-1}(s)ds.
\end{align*}
Thus \[
\int_{\tau=a_i}^{b_i} f(\tau) - g(\tau) d\tau = \int_{s=f(a_i)}^{f(b_i)} g^{-1}(s)-f^{-1}(s)ds.
\]
Notice that if $f > g$ on $[a_i, b_i]$, then $f^{-1} < g^{-1}$ on $[f(a_i), f(b_i)]$. This is due to the following. Given any $y\in [f(a_i), f(b_i)]=[g(a_i), g(b_i)]$, we have $g^{-1}(y)\in[a_i, b_i]$ and $f(g^{-1}(y))> g(g^{-1}(y)) = y = f(f^{-1}(y))$. Thus $g^{-1} > f^{-1}$ since $f$ is strictly increasing. The contrary holds by the same reasoning, \ie~if $f < g$ on $[a_i, b_i]$, then $f^{-1} > g^{-1}$ on $[f(a_i), f(b_i)]$. Therefore, 
\[
\int_{\tau=a_i}^{b_i} |f(\tau) - g(\tau)| d\tau = \int_{s=f(a_i)}^{f(b_i)} |g^{-1}(s)-f^{-1}(s)| ds,
\]
which holds for all intervals $(a_i, b_i)$. Summing over $i$ on both sides, we have
 \[
 \sum_{i=0}^N \int_{\tau=a_i}^{b_i} |f(\tau) - g(\tau)| d\tau =  \sum_{i=0}^N \int_{s=f(a_i)}^{f(b_i)} |g^{-1}(s)-f^{-1}(s)| ds,
 \]
 or equivalently, 
 \[
 \int_{s=0}^1 |f^{-1}(s) - g^{-1}(s)|ds = \int_{\tau=0}^1 |f(\tau) - g(\tau)|d\tau.
 \]
\end{proof}

\end{document}